\documentclass[twoside]{article}

\usepackage[accepted]{aistats2025}
\usepackage[round]{natbib}

\usepackage{url}            
\usepackage[utf8]{inputenc} 
\usepackage[T1]{fontenc}    
\usepackage{hyperref}       
\usepackage{url}            
\usepackage{booktabs}       
\usepackage{amsfonts}       
\usepackage{nicefrac}       
\usepackage{microtype}      
\usepackage[dvipsnames]{xcolor}         

\usepackage[ruled]{algorithm2e}
\usepackage{amsthm}
\usepackage{amsmath}
\usepackage{mathtools}
\usepackage{hyperref}
\usepackage{xr}
\usepackage[capitalise, noabbrev]{cleveref}
\usepackage{caption}
\usepackage{subcaption}
\usepackage{wrapfig}
\usepackage{float}

\usepackage[bottom]{footmisc}

\newcommand{\rbb}{\ensuremath{\mathbb{R}}}
\newcommand{\pbb}{\ensuremath{\mathbb{P}}}

\newcommand{\ebb}{\ensuremath{\mathbb{E}}}

\newcommand{\ie}{i.e.\ }
\newcommand{\eg}{e.g.\ }

\theoremstyle{definition}

\newtheorem{theorem}{Theorem}[section]

\newtheorem*{theorem*}{Theorem}

\theoremstyle{remark}

\begin{document}

%

%

\twocolumn[

\aistatstitle{Score matching for bridges without learning time-reversals}


\aistatsauthor{ Elizabeth L. Baker \And Moritz Schauer \And  Stefan Sommer }

\aistatsaddress{ Dept. of Computer Science\\
University of Copenhagen\\ \And  Dept. of Mathematical Sciences\\
Chalmers University of Technology \\and University of Gothenburg \And Dept. of Computer Science\\
University of Copenhagen\\ } ]

\begin{abstract}
We propose a new algorithm for learning bridged diffusion processes using score-matching methods. 
Our method relies on reversing the dynamics of the forward process and using this to learn a score function, which, via Doob's $h$-transform, yields a bridged diffusion process; that is, a process conditioned on an endpoint. 
In contrast to prior methods, we learn the score term $\nabla_x \log p(t, x; T, y)$ directly, for given $t, y$, completely avoiding first learning a time-reversal. 
We compare the performance of our algorithm with existing methods and see that it outperforms using the (learned) time-reversals to learn the score term. 
The code can be found at 
\url{https://github.com/libbylbaker/forward_bridge}.
\end{abstract}

\section{INTRODUCTION}
Given a diffusion process, we consider conditioning its endpoint to take a specific value or distribution. 
In general, it can be shown by Doob's $h$-transform that a conditioned diffusion process can also be written as a stochastic differential equation (SDE), containing the score term $\nabla_x \log p(t, x; T, y)$, where $p$ is the transition density of the unconditioned SDE. The problem is, this term is in general intractable, since there is no closed form for the transition density. Finding alternative methods to sample from the bridge process has received considerable attention in the past \citep{delyon_simulation_2006,papaspiliopoulos2012importance,bladt2014simple,schauer_guided_2017,heng_simulating_2022}. Diffusion processes have a multitude of applications across numerous areas, where observations need to be included, leading to bridge processes (see \citet{papaspiliopoulos2012importance} for some examples). 
\begin{figure}[ht]
    \centering
\includegraphics[width=0.43\textwidth]{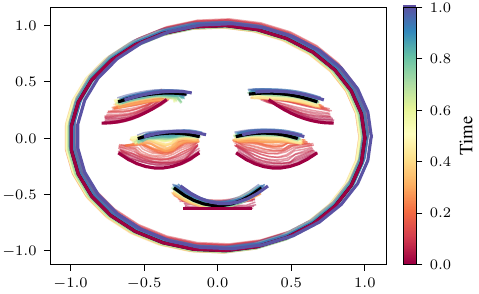}
    \caption{We learn the SDE in \cref{eq: kunita sde} conditioned to hit the shape in black at time $t=1$. We see that, starting the trajectory at the shape in red at time $t=0$, we get a path between the two emojis, where each time point in the path represents a shape.}
    \label{fig:kunita flow}
\end{figure}
Bridges are also strongly linked with generative modelling. There, the aim is effectively to learn a bridge between two given distributions, for example two data distributions, or a data distribution and a normal distribution \citep{de2021diffusion,zhou2024denoising}. 
These methods rely on time-reversals of a linear SDE, where, due to the linearity, the transition densities have a closed form. 
Bridges also are used within shape analysis.
 For example, in medical imaging, to model how shapes of organs change in the face of disease \citep{arnaudon2017stochastic}, or in evolutionary biology to model how shapes of organisms change over time \citep{baker_conditioning_2024}. See \cref{fig:kunita flow} for an example of a conditioned shape process.


\paragraph{Contributions}
We propose a new algorithm to learn bridged diffusion processes by learning the score term occurring in the conditioned process. 
We learn this term by combining the score matching methods proposed in \citet{heng_simulating_2022}, with adjoint diffusions introduced in \citet{milstein_transition_2004}. 
Importantly, these adjoint diffusions can be simulated directly, and replace the need for learning a time-reversal. 
This differs from \citet{heng_simulating_2022}, where in order to learn the score term, they first learn the time-reversed diffusion and then train a second time on the learned trajectories to get the forward-in-time bridge. Our work negates the need for learning the time-reversal, meaning we only need train once, effectively halving the necessary training time. 
We compare our method with other existing methods and evaluate it on a range of linear and non-linear SDEs.

\section{BACKGROUND}

We are interested in conditioning an SDE to take a certain value at its end time.
The conditioned process can again be written as an SDE via Doob's $h$-transform which we describe in \cref{sec: doobs}. 
However, this SDE has an intractable term, the score $\nabla_x \log p(t,x;T,y)$. 
In \cref{sec: reversals} we discuss the adjoint of the SDE which, up to a correction term $\varpi$, has the transition density $q(t',y; t,  \cdot) = \varpi(t', y; t)^{-1}  p(t, \cdot; t', y)$ when started in $y$ at time $t'$. 
In the proposed method, we use the adjoint processes to directly learn the term $\nabla_x \log p(t, x; T, y)$, used for forward bridges.
 Score learning provides a method for learning a similar but different term $\nabla_{x} \log p(0, x_0; t, x)$ and is discussed in \cref{sec: score matching}.

\subsection{Doob's \texorpdfstring{$h$}{h}-transform}
\label{sec: doobs}

Consider an SDE with $X(0) = x_0 \in \rbb^d$ defined as
\begin{align}\label{eq: unconditioned process}
    \mathrm{d}X(t) = f(t, X(t))\mathrm{d}t + \sigma(t, X(t))\mathrm{d}W(t),
\end{align}
where $f\colon[0, T] \times \rbb^d \to \rbb^d, \sigma\colon[0, T] \times \rbb^d \to \rbb^{d \times m}$ and $W(t)$ is a $m$-dimensional Wiener process. We assume that $f$ and $\sigma$ satisfy a global Lipschitz condition so that \cref{eq: unconditioned process} has a Markov solution \citep[Theorem 19.23]{schilling_brownian_2014}. Assume further that $X(t)$ has a transition density $p(t, x; t', x')$ with $t'>t$, satisfying
\begin{align}
   \ebb\left[X(t')\in A \mid X(t)=x\right] = \int_{A} p(t, x; t', x') \mathrm{d}x'. 
\end{align}
It is well known that $X(t)$ conditioned on $X(T) = y$ for a fixed $y$, satisfies another SDE, namely
\begin{align}\label{eq: conditioned process}
\begin{split}    
        \mathrm{d}\Bar{X}(t) &= 
        f(t, \Bar{X}(t)) \mathrm{d}t\\ 
        &+ {(\sigma\sigma^\top)(t, \Bar{X}(t))}\nabla \log p(t, \Bar{X}(t); T, y)
        \mathrm{d}t \\
        &+ \sigma(t, \Bar{X}(t))\mathrm{d}W(t), 
\end{split}
\end{align}
with $\Bar{X}(0) = x_0$ and where $\nabla$ is the gradient with respect to the argument $\Bar{X}(t)$ \citep[Chapter 7.5]{sarkka_applied_2019}.  
However, the score $\nabla_x \log p(t, x; T, y)$ is intractable for non-linear SDEs. 

\subsection{Adjoint processes with reversed dynamics}\label{sec: reversals}

Here we motivate the adjoint process $\{Y(t), \mathcal{Y}(t)\}$ of the unconditioned process $X(t)$ which receives its name from its relation to the adjoint operator of the infinitesimal generator of $X$ seen as a Markov process.
This process is properly defined in \cref{eq: reverse and correction}. 
Intuitively, the adjoint process has the reversed dynamics of the unconditioned process: for example, the adjoint $Y(t)$ of a Wiener process with drift $\mu$, i.e.~$W(t) + t \mu$, is a Wiener process with drift $-\mu$.
In this sense, it is a reversal, but not a time-reversal as in \citet{haussmann1986time}, since it does not satisfy $Y(T-t) = X(t)$.

The theory presented in this section is based on \citet{milstein_transition_2004, bayer_simulation_2013}
 and is also developed in \citet[Section 11]{v2_2020automatic}. 
The significance of the adjoint process is that for fixed $T, y$, we can use it to access $p(t, x; T, y)$ and therefore the score $\nabla \log p(t, x; T, y)$. 
To compare to the forward process, fix $x_0, t_0, y, T$ and let $g$ be a real-valued function. Then 
\begin{align}\label{eq: expectation of reverse}
    \ebb[g(X(t))] &= \int g(x) p(t_0, x_0; t, x) \mathrm{d}x\\ 
\ebb[g(Y_{t, y}(T))\mathcal{Y}_{t, y}(T)] &= \int g(x)p(t, x; T, y) \mathrm{d}x.
\end{align}
Here, $\mathcal Y$ is a weight process accounting for the fact that adjoint dynamics in general are not Markovian necessitating a Feynman-Kac representation. 
The processes arise by using the Fokker-Planck equation for the transition density $p$, and then multiplying this by $g(x)$ and integrating with respect to $x$. Doing so leads to a Feynman-Kac representation for $\int g(x)p(0, x_0; t, x) \mathrm{d}x$, with SDE $Y$ and weight function $\mathcal{Y}$. For details we refer to \cite{milstein_transition_2004}.

For fixed $y, t_0, T$, with $t_0\leq t' \leq T$, and $T_0 = T +t_0$, the adjoint process 
$\{Y_{t_0, y}(t), \mathcal{Y}_{t_0, y}(t)\}$
is defined as follows:
\begin{subequations}\label{eq: reverse and correction}
\begin{align}
    \mathrm{d}Y &= \alpha(t, Y) \mathrm{d}t + \Tilde{\sigma}(t, Y)\mathrm{d}\Tilde{W}(t),  &Y(t_0) = y,\label{eq: reverse}\\ 
    \mathrm{d}\mathcal{Y} &= c(t, Y) \mathcal{Y}\mathrm{d}t, &\mathcal{Y}(t_0) = 1,\label{eq: correction} 
\end{align}    
\end{subequations}
where
\begin{subequations}   
\begin{align}
\begin{split}    
   \alpha^i(t', x) &\coloneqq \sum_{j=1}^{d} \frac{\partial}{\partial x^j}(\sigma\sigma^\top)^{ij}(T_0-t', x)
   \\
   &~~~~~-f^i(T_0-t', x),
\end{split}
   \\
    \tilde{\sigma}(t', x) &\coloneqq \sigma(T_0-t', x),
    \\
    \begin{split}     
    c(t', x) &\coloneqq \frac{1}{2}\sum_{i,j =1}^{d} \frac{\partial^2(\sigma\sigma^\top)^{ij}}{\partial x^i \partial x^j}(T_0-t', x) 
    \\ 
    &~~~~~-\sum_{i=1}^{d}\frac{\partial f^i}{\partial x^i}(T_0-t', x).
\end{split}
\end{align}
\end{subequations}
From now on, we denote $Y_y \coloneqq Y_{0, y}$ and $\mathcal{Y}_{y} \coloneqq \mathcal{Y}_{0, y}$, but the following also holds for $t_0 \neq 0$. 
Crucially, $Y_{y}$ and $\mathcal{Y}_{y}$ can be computed explicitly, either by hand or using automatic differentiation. Therefore, we can sample from the adjoint process using, for example, the Euler-Maruyama scheme. 
Moreover, following \citet[Theorem 3.3]{bayer_simulation_2013}, we can introduce more time points. For a time grid $\mathcal{T}=\{0=t_0 < t_1 < \cdots < t_L =T\}$, and setting $\hat{t}_i = T- t_{L-i}$ and $y_{L+1} = y$, it holds 
\begin{align}
\begin{split}
    &\ebb[f(Y_{y}(\hat{t}_L), \cdots, Y_{y}(\hat{t}_1))\mathcal{Y}_{y}(T)]\\
    &= \int_{\rbb^{d\times L}} f(y_1, \cdots, y_L)\prod_{i=1}^L p(t_{i-1}, y_i, t_i, y_{i+1})\mathrm{d}y_i. 
    \end{split}
\end{align}
In particular, we can now represent the score $s(t, x)$ as the gradient of the log density $q({T-t}, \cdot)$,
where $q({T-t}, \cdot) = p(t, x; T, y)$
is the marginal distribution of $Y_y(T-t)$ weighted by $\mathcal Y_y$, and $\hat t \in [0,T]$ by
\begin{align}    
\int_A f(x) q({\hat t}, x) \mathrm{d}x = \ebb[f(Y_{y}(\hat t)) \mathcal{Y}_{y}(\hat{t}_L)].
\end{align}

\subsection{Score matching bridge sampling}\label{sec: score matching}
In score matching bridge sampling, the task is to learn a neural approximation $s_\theta(t, x)$ to the score $ s(t, x) = \nabla_x \log p(t, x; T, y)$. In our case, 
samples from $Y_y(T-t)$ weighted with  $\mathcal{Y}_{y}(T-t)$ can be used in training, as they represent the unnormalised density $q(T-t, \cdot) = p(t, \cdot; T, v)$.
This precisely mimics the use of samples of the noising process in generative diffusion models with  $Y_y({\hat t}),  \mathcal{Y}_{y}(\hat t)$ in the role of the noising process.

To wit, vanilla score matching as introduced in  \citet{JMLR:v6:hyvarinen05a} takes the form, for $\hat t = T - t$
\begin{align}
&\min_{\theta} \int_0^T\int q({\hat t},x) \sum^N_{i = 1}S_i\mathrm{d} x \mathrm{d}t,
\\
S_i &\coloneqq
\frac{\partial}{\partial x_i} s_{\theta, i}(t, x) + \frac{1}{2} s_{\theta, i}(t, x)^2, 
\end{align}
where we introduce $S_i$ for notational brevity.
Note in particular that knowledge of the transition density is not required, however, computing the divergence term is expensive.
To avoid computing the divergence term when learning $\nabla_{x} \log p(0, x_0; t, x)$, instead one can use a similar loss objective to that proposed by \citet{vincent2011connection}
\begin{align}\label{eq: intractable loss}
    \mathcal{L}(\theta) = \int_0^T \ebb[\|s_\theta(t, X_t) - s(t, X_t)\|^2]\mathrm{d}t,
\end{align}
where $s(t, x) = \nabla_{x} \log p(0, x_0; t, x)$. We instead will take the expectation with respect to conditioned trajectories, and will let $s(t, x) \coloneqq \nabla \log p(t, x; T, y)$.
In and of itself, this optimisation problem is still intractable. However, \citet{heng_simulating_2022} show that for $C$ independent of $\theta$
\begin{subequations}    
\begin{align}
\mathcal{L}(\theta) &= 
    C +
    \sum_{l=0}^{L-1}
    \int_{t_l}^{t_{l+1}}\ebb [
    \|s_{\theta}(t, X_t) - s_l\|^2_{\sigma\sigma^\top} ]\mathrm{d}t,\\
s_l &= \nabla \log p(t_{l}, X_{t_{l}}; t, X_t),
\end{align}
\end{subequations}
 where $s_l$ is used to simplify the notation, and $\|x\|_A = \|A^{1/2}x\|$ is the weighted norm for $A$ a positive definite matrix.
For small time steps, the transition density $p(t_{l}, X_{t_{l}}; t, X_t)$ can be estimated using the Euler-Maruyama scheme.

\section{RELATED WORK}
In this section, we survey the existing methods for computing bridge processes. 
We concentrate specifically on methods that can be applied to any diffusion process to find the forward bridge. 
For this reason, we deem usual score matching methods as out of scope, since the main motivation of these methods is to find a time-reversal of a specific diffusion, usually an Ornstein-Uhlenbeck, where the score function can be written in closed form.

\subsection{Score matching methods}

Most related to our work is \citet{heng_simulating_2022}'s, who first learn the time-reversal of the bridge process. 
They observe that this is equivalent to learning the term $\nabla_x \log p(0, x_0; t, x)$. Repeating the process a second time -- now using the learned time-reversals in place of unconditioned forward processes -- gives a forward-in-time bridge.
Moreover, they propose a loss function that works with non-linear SDEs, where the transition density is not tractable.
We build on \cite{heng_simulating_2022}'s work by using their loss function for bridging diffusion bridges. However, our implementation requires only half the compute as follows: Instead of training on time-reversals (which must be learned), we train using adjoint processes. The adjoint processes can be written in closed form and therefore we may directly simulate from them. Using these, we learn the term $\nabla \log p(t, x; T, y)$ (and thereby the bridged process) directly, without needing to train twice.

Recently, \citet{singhal2024whats} consider
using non-linear SDEs
instead of linear ones within generative modelling, meaning
the transition densities for the SDE
are also unknown. However,
instead of using Euler-Maruyama
approximations for the transition
kernels as in \citet{heng_simulating_2022}, they
consider other higher-order approximations.
Due to the non-linear SDEs this is closer to our problem, however they find time-reversals and not forward bridges, which is the focus of this work.

\subsection{MCMC methods}
One method of sampling from the bridge processes is to sample from a second approximate bridge, where the ratio between the true bridge and the approximate bridge is known. For an unconditional process as in \cref{eq: unconditioned process}, the approximate bridges have form
\begin{align}
\begin{split}
\mathrm{d}X(t) &= f(t, X(t))\mathrm{d}t + s(t, X(t))\mathrm{d}t\\ 
&~~~~+ \sigma(t, X(t))\mathrm{d}B(t),
\end{split}
\end{align}
where different authors have considered different options for the function $s$. In \citet{pedersen1995consistency} $s=0$, that is, the samples are taken from the original unconditioned process. \citet{delyon_simulation_2006} instead take the Brownian bridge drift term. More recently, \citet{schauer_guided_2017} set $s$ to be the term arising from conditioning an auxiliary process with a known transition density. Then, one can derive a ratio between the path probability of the true bridge and the bridge approximation and can therefore use Markov chain Monte Carlo methods to sample from the true distribution.

\subsection{Adjoint methods}

\citet{milstein_transition_2004} estimate the transition density of an SDE by introducing adjoint processes. Simulation happens both forward from the start state and backwards from the end state. When connected, this is called a forward-reverse representation, and it gives rise to a Monte-Carlo estimator. \citet{bayer_simulation_2013} expand this to estimating expectations of functions of multiple time-steps of conditioned SDEs, again with Monte-Carlo methods. We use the probabilistic representations with adjoints they introduce, however we use these to construct a loss function to learn $\nabla_x \log p(t, x; T, y)$ for varying $t, x, y$. Thereby, we can sample directly from the conditioned trajectories. 

\section{METHOD}

This section presents our algorithm to find the score function $\nabla_x \log p(t, x; T, y)$ that will be used to simulate the bridge process. 
In \cref{sec: method loss function}, we derive the main loss function for learning the score function and show it minimises the Kullback-Leibler divergence between the conditioned path and the path measure induced by the learnt conditioned process. In \cref{sec: computing loss}, we provide details on how we minimise the loss function.
In \cref{sec: method multiple end points}, we consider varying values of $y$, and the case that the end point $y$ is forced to follow a given distribution, instead of being a fixed point.
Proofs can be found in \cref{Appendix: Proofs}.

\subsection{A loss function}\label{sec: method loss function}

In order to sample the bridged SDE, we learn the score term $\nabla_x \log p(t, x; T, y)$ for fixed $T, y$. 
Note that, although similar, this is different from the term $\nabla_{x} \log p(0, x_0; t, x)$ that is learnt for time-reversed processes, as in score-matching for diffusion samplers. 
Indeed, score-matching methods rely on the SDE $X_t$ having $p(0, x_0; t, x)$ as a transition density, and therefore by sampling from the process $X_t$, we are sampling from the distribution of the process at time $t$. 
This is not true for $p(t, x; T, y)$. 
Our proposal is instead to sample from the adjoint stochastic process, where we can recover the transitions $p(t, x; T, y)$.

Given a function $s\colon[0, T] \times \rbb^d \to \rbb^d$, we define the SDE $X^s$ with path distribution $\mathbb{Q}^s$ as
\begin{align}\label{eq: X^s}
\begin{split}  
    \mathrm{d}X^s(t) 
    &
    = f(t, X^s(t))\mathrm{d}t 
    \\
    &\quad\quad
    + (\sigma \sigma^\top)(t, X^s(t))s(t, X^s(t))\mathrm{d}t 
    \\
    &\quad\quad
    + \sigma(t, X^s(t))\mathrm{d}W(t).
    \end{split}
\end{align}

In what follows, let $K_\epsilon(x, y)$ be a function, such that $K_\epsilon (x, y) \to \delta_0(x-y)$ as $\epsilon \to 0$, for $\delta_0$ the dirac delta function. For fixed $y, T$ with $Y \coloneqq Y_{0, y}, \mathcal{Y} \coloneqq \mathcal{Y}_{0, y}$ (see \cref{eq: reverse,eq: correction}) and $\hat{t} \coloneqq T-t$ we introduce
\begin{subequations}
\begin{align}\label{eq: loss fn x0}
    \mathcal{L}(s, x_0) 
    &\coloneqq \ebb\left[\frac{\lim_{\epsilon \to 0}\mathcal{Y}(T) K_\epsilon(Y(T), x_0)}{p(0, x_0; T, y)}\sum_{\ell=1}^L I_\ell\right],
    \\ \label{eq: Il}
    I_\ell &= \int_{t_{\ell-1}}^{t_\ell} 
    \|s(t, Y(\hat{t}))
    - s_\ell(t, Y(\hat{t}))\|_{\sigma\sigma^\top}^2\mathrm{d}t,
\\
    s_\ell(t, Y(\hat{t})) &= \nabla \log p(t, Y(\hat{t}); t_\ell, Y(\hat{t_\ell})),
\end{align}
\end{subequations}
where $s_\ell, I_\ell$ have been used to make the notation more compact.
By minimising this, we are minimising a Kullback-Leibler divergence:

\begin{theorem} \label{thm: loss function}
Let $X$ be a diffusion process as defined in \cref{eq: unconditioned process} and suppose further that $f, \sigma$ are $C^2$ and that $X$ admits $C^2$ transition densities. Let $\Bar{X}$ be the conditioned diffusion satisfying \cref{eq: conditioned process}. Let $\Bar{\pbb}$
be the path probability of 
$\Bar{X}$ and 
$\mathbb{Q}^s$ the path probability of 
$X^s$ in \cref{eq: X^s} for a function $s$. 
Then for a given starting point $x_0$
\begin{align}
    \arg\inf_{s} \text{KL}(\Bar{\pbb} \parallel \mathbb{Q}^s) = \arg\inf_{s} \mathcal{L}(s, x_0),
\end{align}
where the infimum is over functions 
$s\colon [0, T]\times \rbb^d \to \rbb^d$.
\end{theorem}
\begin{proof}
 First, we use the adjoint SDE and \citet[Theorem 3.7]{bayer_simulation_2013} to write the expectations as integrals over the transition densities $p(t, x; T, y)$.
 Secondly, as proposed in \citet{heng_simulating_2022}, we use the Chapman-Kolmogorov equation to split the integration into smaller, approximable parts. A full proof is given in \cref{proof: kl loss function}.
\end{proof}

We have shown that finding $s$ that minimises the loss function in \cref{eq: loss fn}, also minimises Kullback-Leibler divergence between the conditioned path probability space 
and the path space $\mathbb{Q}^s$. The difference is that it is possible to approximate $\mathcal{L}(s)$. For example, we can take $\epsilon$ small, and for $t'-t$ small we can approximate $\nabla_x \log p(t, x; t', x')$. 
Next, we remove the dependence of the loss function on the start point.

\begin{theorem}\label{thm: independent loss fn}
We assume the same setup as in \cref{thm: loss function}. Define $\Bar{\mathcal{L}}(s) \coloneqq \int\mathcal{L}(s, x_0) \mathrm{d}x_0$. Then, letting $I_\ell$ be defined as in \cref{eq: Il}:
\begin{subequations}
\begin{align}\label{eq: loss fn}
        \Bar{\mathcal{L}}(s) &= \mathbb{E} \left[\frac{\mathcal{Y}(T)}{p(0, Y(T); T, y)}
        \sum_{l=1}^L I_\ell
    \right].
\end{align}
\end{subequations}

\end{theorem}

\begin{proof}
    This is a consequence of the integration of $K_{\epsilon}(Y(T), x_0)$. See \cref{proof: x_0 integral} for details.
\end{proof}
\begin{figure*}[ht]
    \centering
\includegraphics[width=0.95\textwidth]{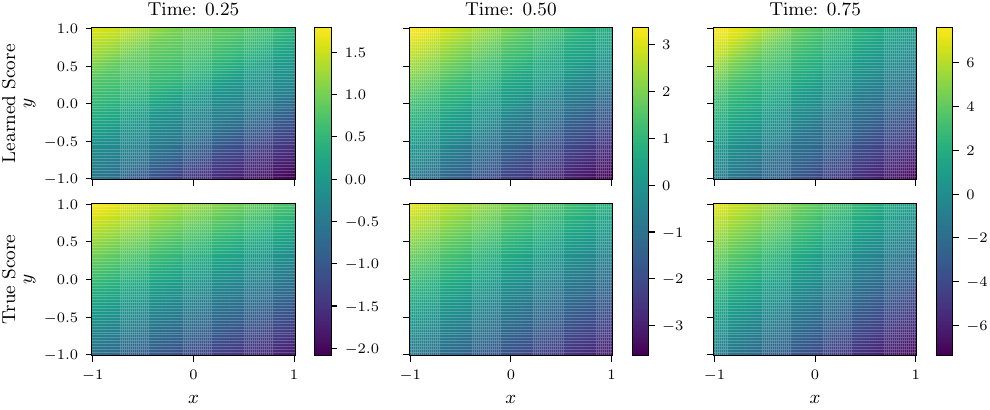}
    \caption{For varying $x, y$ and different times $t \in [0.25, 0.5, 0.75]$, we plot the true score $\nabla_x \log p(t, x; 1, y)$ and the learned score $s_\theta(t, x; y)$, as described in \cref{sec: ou experiment}. The learned score takes a fixed end time $T=1$, but takes both $x$ and $y$ as input values. 
    The absolute error between the true and learned score is reported in \cref{fig: ou error}.
}
    \label{fig:varied y error}
\end{figure*}
Finally, we define a simplified loss function, which we use in the experiments. It is a weighted version of $\Bar{\mathcal{L}}(s)$:
\begin{align}
\begin{split}
    \mathcal{L}(s) &\coloneqq \mathbb{E} \left[
        \sum_{\ell=1}^L
         I_\ell
    \right]
\end{split}
\end{align}
where $I_\ell$ is as in \cref{eq: Il}.

\subsection{Optimising the loss function}\label{sec: computing loss}

To solve the optimisation problem, we let $s$ be parameterised by a neural network. We simulate batches of trajectories of the adjoint process $\{Y(t), \mathcal{Y}(t)\}$ using the Euler-Maruyama algorithm.
Like \citet{heng_simulating_2022}, we approximate $p(t, x; t+\epsilon, y)$ as a Gaussian, so that 
$\nabla \log p(t, x; t + \epsilon, y) \approx g(t, x; t+\epsilon, y)$, where $g$ is a Euler-Maruyama step, \ie for $(\sigma \sigma^\top)\coloneqq \Sigma$
\begin{align}
g(t, x, t', x') = \frac{1}{\Delta t} \Sigma(t, x)^{-1}(x' -x -\Delta t f(t, x)).
\end{align}
We then approximate the loss function as
\begin{subequations}
\begin{align}\label{eq: approx loss fn}
    \mathcal{L}(s_\theta) &\approx 
    \frac{ \Delta t}{N} \sum_{Y \in B_N} \sum_{\ell =0}^{L-1} \| T_{\ell} \|_{\sigma\sigma^\top}^2,
    \\
    T_{\ell} &\coloneqq s_\theta (t_{\ell }, Y_\ell) - g(t_\ell , Y_\ell; t_{\ell +1}, Y_{\ell+1}),
\end{align}
\end{subequations}
where $Y_\ell \coloneqq Y(T-t_\ell)$ and $B_N$ is a batch of $N$ simulations of $Y(t)$ started at the value $y$ on which we are conditioning.
\cref{algo: fixed y} gives the exact steps.

After using \cref{algo: fixed y} to learn $s_\theta$, we can then sample from the conditioned SDE
\begin{align}
\begin{split}   
\mathrm{d}X(t) &= (\sigma\sigma^\top)(t, X(t))s_\theta(t, X(t))\mathrm{d}t
\\
&~~~~+ f(t, X(t))\mathrm{d}t +\sigma(t, X(t))\mathrm{d}W(t), 
\end{split}
\end{align}
using, for example, the Euler-Maruyama algorithm.

\subsection{Training on multiple end points}\label{sec: method multiple end points}

We can extend to training on multiple end points in two ways. The first is where the end point $y$ is not fixed, rather sampled from some distribution $\pi_T$. 
That is we are conditioning by forcing the SDE to have a given distribution at the end time as in \cite{baudoin2002conditioned} (see also \citet[Example 2.4]{corstanje2023conditioning}). In this case we learn the term $\nabla \log h(t, x)$, where $h(t, x) \coloneqq \int
\frac{p(t, x; T, y)}{p(0, x_0; T, y)} 
\pi_T(\mathrm{d}y)$.
The second, is where we learn fixed $y$ for multiple $y$ values at once; that is for fixed $T$ and \textit{variable} $y$, we learn $s_\theta(t, x; y) \approx \nabla_x \log p(t, x; T, y)$. Both of these are straightforward extensions to the method, also in \cite{heng_simulating_2022}.

\paragraph{Distributions of endpoints}
To condition on a distribution $\pi_T$ instead of a fixed point $y$, we integrate over the loss function for different values of $y$. In practice, all that changes is instead of simulating trajectories $Y, \mathcal{Y}$ starting from the same value $y$, we now also sample the start value $y \sim \pi_T$.
Letting $Y_{y, \ell} \coloneqq Y_{y}(T-t_\ell)$, the loss in \cref{eq: approx loss fn} function becomes 
\begin{subequations}
\begin{align}
\mathcal{L}(s_\theta) &\approx
\frac{\Delta t}{N} \sum_{y \sim \pi_T} \sum_{\ell =0}^{L-1} \|T_\ell\|_{\sigma\sigma^\top}^2,
 \\
 T_{\ell} &\coloneqq s_\theta (t_{\ell }, Y_{y, \ell}) - g(t_\ell , Y_{y, \ell}; t_{\ell +1}, Y_{y, \ell+1}).
\end{align}
\end{subequations}

\paragraph{Multiple fixed endpoints}\label{sec:multiple fixed endpts}

For varied $y \in B \subseteq \rbb^d$, we integrate the loss function in \cref{eq: loss fn} over $y$ and take the loss over functions $s\colon [0, T] \times \rbb^d \times \rbb^d \to \rbb^d$. Then we take a batch $B_N$ of $N$ uniformly distributed $y \in B$.
In practice, the loss function that we optimise changes from \cref{eq: approx loss fn} to
\begin{subequations}  
\begin{align}
    \mathcal{L}(s_\theta) &\approx 
    \frac{\Delta t}{N} \sum_{y \in B_N}\sum_{\ell =0}^{L-1} 
    \|T_\ell\|_{\sigma\sigma^\top}^2,
\\
T_\ell &\coloneqq s_\theta (t_{\ell }, Y_{y, \ell}, y) - g(t_\ell , Y_{y, \ell}; t_{\ell +1}, Y_{y, \ell+1}),
\end{align}
\end{subequations}
where we have used $Y_{y, \ell} := Y_y(T-t_\ell)$, to mean the adjoint process $Y$ started from $y$ at time $0$.

\section{EXPERIMENTS}

\begin{figure}[ht]
\centering
\includegraphics[width=0.42\textwidth]{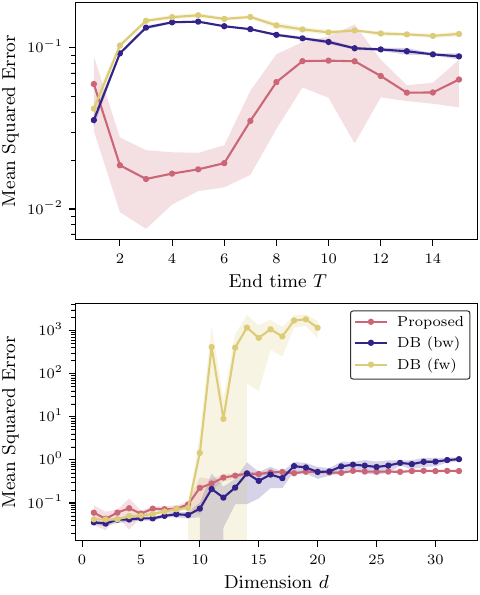}
\caption{
We plot the square error between the true and learned score averaged over times from $t=0$ to $t=1$.
For the proposed method and DB (fw) we learn $\nabla_{x} \log (t, x; T, y)$, and for DB (bw) we learn $\nabla_x \log (0, x_0; t, x)$.
\emph{Top:} Trained for $x_0=1$, $y=1$, with $1$-dimension and varying end times. 
\emph{Bottom:} Trained for $x_0=1.0$, $y=1.0$, with $T=1.0$. See \ref{sec: ou experiment} for details.}
\label{fig:dim_comparison}
\end{figure}
\begin{figure}[ht]
    \centering
\includegraphics[width=0.4\textwidth]{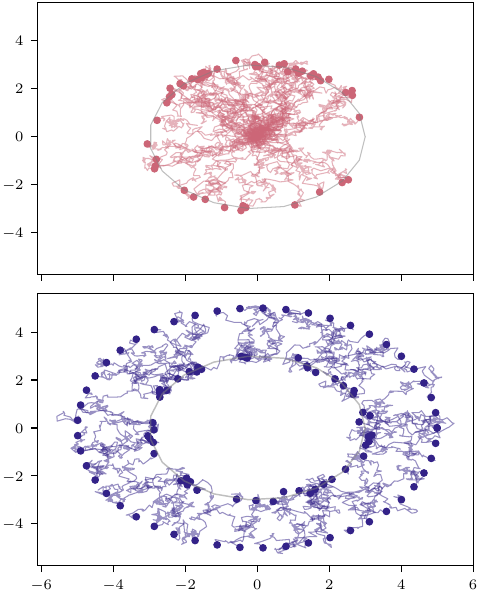}
    \caption{We train a 2D Wiener process process to hit a circle of radius $3$ at time $T=1.0$. On the top we plot $20$ trajectories started from the centre, and on the bottom we plot $20$ trajectories started from evenly spaced points on a circle of radius $5$.}
    \label{fig: bm distributed}
\end{figure}

We compare only to methods considering the problem where we are \textit{given} an SDE that we condition on given boundary conditions. 
This is different to the setup of generative modelling or denoising diffusion bridge models, where the SDE can be chosen freely. 
In this case, a natural choice is a linear SDE, which has explicit transition densities. Further experiment details can be found in \cref{Appendix: experiments}. Code is available at \url{https://github.com/libbylbaker/forward_bridge}.

\subsection{Ornstein--Uhlenbeck}\label{sec: ou experiment}

First, we condition an Ornstein-Uhlenbeck process. This is a good experiment to evaluate the method, since the true score is known, so we can compute the error. 
Explicitly, we consider 
\begin{align}\label{eq: OU unconditioned}
    \mathrm{d}X(t) = - X(t) \mathrm{d}t + \mathrm{d}W(t).
\end{align}
The formulas for the adjoint process and the score function can be found in \cref{Appendix: experiments ou}.

In our first experiment we learn $s_\theta(t, x; y) \approx \nabla \log p(t, x; T, y)$, for $y \in [-1.0, 1.0]$, as in \cref{sec:multiple fixed endpts}.
We plot the true and learned score in \cref{fig:varied y error}. In \cref{fig: ou error,fig: ou data distribution} we also plot the absolute error and the generated trajectories. We see that at earlier times the generated trajectories are less dispersed, thereby leading to higher errors in the score when the difference between $x$ and $y$ is larger.

Next, we compare our results to \citet{heng_simulating_2022} (using the authors' code \footnote{\texttt{https://github.com/jeremyhengjm/DiffusionBridge}.}).
We want to show the errors that occur when using the adjoint processes to learn the bridge, are similar to the errors using the forward processes to learn the time-reversal, and either similar or better than using the time-reversals to learn the forward bridge.

 In \cref{fig:dim_comparison}, we plot the squared error averaged over time points from $t=0$ to $t=1$ for various dimensions $d$ with $x_0=1.0$, and various end times $T$, with $d=1$. 
We compare this to the forward and time-reversed (or backwards) diffusion bridges from \citet{heng_simulating_2022}. 
The time-reversed, DB (bw), learns $\nabla_x \log p(0, x_0; t, x)$ instead of $\nabla_x \log p(t, x; T, y)$. The forward, DB (fw), leverages the time-reversal to learn $\nabla_x \log p(t, x; T, y)$.

As shown in \cref{fig:dim_comparison}, the time-reversed bridges (DB (bw)) demonstrate a similar error to the proposed approach. However, in learning the forward bridge, the proposed method is not only twice as efficient as DB (fw) -- requiring only half the number of trajectories due to being trained once -- but also exhibits greater stability in higher dimensions.

Finally, to see how it scales to slightly higher dimensions we train our proposed method under the same setup, also for 5 different random seeds. We report the mean squared error and standard deviation results in \cref{tab: ou high dim}.

\begin{table}[h]
\caption{Ornstein-Uhlenbeck Errors} \label{tab: ou high dim}
\begin{center}
\begin{tabular}{rrr}
   \textbf{Dim.} &      \textbf{MSE} &         \textbf{SD} \\
\hline
     32 & 0.553 & 0.008 \\
     50 & 0.548 & 0.019  \\
    100 & 0.567 & 0.003 \\
    200 & 0.566  & 0.003  \\
\end{tabular}
\end{center}
\end{table}

\subsection{Wiener process on distributed endpoints}

We may also wish to condition on a distribution, rather than fixed endpoints. To demonstrate the method in this situation, we condition a Wiener process on hitting a circle. That is, we take the SDE $\mathrm{d}X(t) = \mathrm{d}W(t)$, and we sample the end points $y_0$ uniformly from a circle of radius $3.0$. In \cref{fig: bm distributed} (left), we plot $20$ trajectories, each starting from the origin, and conditioned to hit the circle at time $T=1.0$. In \cref{fig: bm distributed} (right) we plot $20$ trajectories, this time starting from evenly sampled initial values on a circle of radius $5.0$. Note that both figures use the same trained score function. We only need to train once, since the training is independent of the initial value $x_0$, and only depends on the final value or distribution on which we are conditioning. 

\subsection{Cell model}\label{sec: cell}
So far we have only considered linear models. To test on non-linear models, we next consider the cell differentiation and development model from \citet{wang2011quantifying}, also considered by \citet{heng_simulating_2022}.
The SDE is defined as
\begin{align}\label{eq: cell sde}
\mathrm{d}X_1(t) &= f(X_1(t), X_2(t))
 \mathrm{d}t + \sigma \mathrm{d}W (t)
 \\ 
 \mathrm{d} X_2 (t) &= f (X_2 (t), X_1 (t)) \mathrm{d} t + \sigma \mathrm{d}W(t),
 \\
f(x, y) &= \frac{x(t)^4}{2^{-4} + x(t)^4} 
+ \frac{2^{-4}}{2^{-4} + y(t)^4}
- x(t)
\end{align}
and describes the process of cell differentiation into three cell types. 
In \cref{fig: cell forward} we plot some trajectories
of the unconditioned process for illustration.
We compare to the time-reversed bridge and the forward bridge of \citet{heng_simulating_2022} and the guided bridge method from \citet{schauer_guided_2017}. We choose to compare with these methods, since \citet{heng_simulating_2022} have already shown favourable results compared to \citet{pedersen1995consistency,durham2002,delyon_simulation_2006}. 

The true bridge is unknown, so instead we compare against trajectories from the unconditioned process with similar the boundary values. In \cref{fig: cell comparison}, we plot in grey trajectories from the unconditioned process satisfying $x_1(T) \in (1.3, 1.7), x_2(T) \in (0.0, 0.4)$ at time $T=2.0$. We then plot trajectories starting at $(0.1, 0.1)$, conditioned on $(1.5, 0.2)$ at time $T=2.0$, via the proposed method, the forward bridge (DB (fw)) of \citet{heng_simulating_2022} and with MCMC using guided proposals as in \citet{schauer_guided_2017}. 
For the time-reversed bridge, the trajectories are simulated from the end point $(1.5, 2.0)$, and run backwards in time, being conditioned to hit the point $(0.1, 0.1)$. 
The trajectories from the time-reversed SDE are then used to train the forward bridge \citep[Section 2.3]{heng_simulating_2022}.

We see in \cref{fig: cell comparison} that the proposed method, the time-reversed bridge and the forward bridge produce similar trajectories that follow the correct data distribution.
However, again the proposed method need only be trained once, whereas the forward bridge was trained on the time-reversed trajectories thereby needing double the amount of training. 

\begin{figure*}[h]
    \centering    \includegraphics{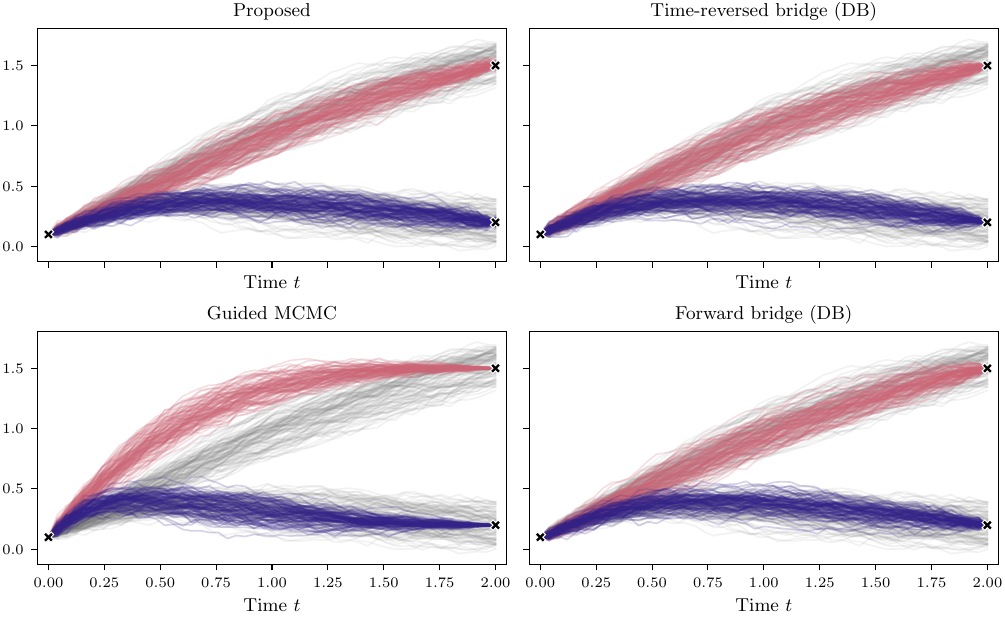}
    \caption{We plot $100$ conditioned trajectories with the methods \textit{Proposed, Guided MCMC} and \textit{Forward bridge (DB)} and the time-reversed bridge with  \textit{Time-reversed bridge}.  In grey, we plot some unconditioned trajectories that end in a given window about the point $(1.5, 0.2)$. For details see \ref{sec: cell}.}
    \label{fig: cell comparison}
\end{figure*}

\subsection{Shape models}

We continue increasing the difficulty of the problems, and apply our proposed method to stochastic shape analysis \citep{Arnaudon_Holm_Sommer_2023}. 
This is an interesting 
test since it involves conditioning SDEs on shape spaces, where the stochastic processes are non-linear and have non-trivial diffusion functions. 
Such a model is found in evolutionary biology when studying how the morphometry of a species changes as the species evolves \citep{baker_conditioning_2024,yang2024simulating}. Therefore, this experiment will give insights into what we can expect when applying the method to evolutionary biology, though we apply it to non-biological shapes.

For a set of points $\{x_i\}_{i=1}^n \subset \rbb^2$ representing the discretisation of a shape, we consider the SDE
\begin{align}\label{eq: kunita sde}
    \mathrm{d}X_t = \sum_{i=1}^m k(X(t), z_i)\mathrm{d}W_i(t),
\end{align}
with the initial constraint $X(0)=\{x_i\}_{i=1}^n \in \rbb^{n\times 2}$,
where $W_i$ are independent Wiener processes in $\rbb^2$, $\{z_i\}_{i=1}^m \subset \rbb^2$ is a grid on a compact subset of $\rbb^2$ and $k$ is the Gaussian function defined as $k(x, y) = \alpha \exp{-\frac{\|x-y\|^2}{2\sigma^2}}$ for some scalars $\sigma, \alpha$. 
This SDE represents the position of the points $x_i$ at time $t$.

We use the neural network architecture of \citet{yang2024simulating}. They use Fourier neural operators and learn the term $\nabla \log p(0, x_0; t, x_t),$ using the method proposed by \citet{heng_simulating_2022} to get a time-reversal of $Y_t$. 
This is applied to butterfly data to bridge between the outlines of butterflies. 
Instead, we apply the proposed method to the outlines of two emojis, conditioning such that the points reach a ``smiley face'' at time $t=1$. 

In \cref{fig:kunita flow} we plot one trajectory of the SDE \cref{eq: kunita sde}. 
At each time point, $X(t) \in \rbb^{n\times 2}$ can be interpreted as a stochastic change in the original shape. 
We have conditioned the SDE in \cref{eq: kunita difference} using \cref{algo: fixed y} such that at time $t=1$, the configuration of points is the smiley face shown in blue. For this trajectory, we start it from the sad face shown in red.
In this case, we cannot evaluate the learnt score directly, since the true score is unknown. 
However, we see that indeed the sample paths end at the wanted boundary points as expected, even for non-trivial diffusion functions and trivially chosen shapes.

\section{PROS, CONS, AND CONCLUDING REMARKS}

We believe the main drawback of this method is conditioning on unlikely events, that are not covered in the data. 
An obvious example would be conditioning a 1-dimensional Wiener process $W(t), W(0)=0$ such that $W(1)=100$. In this case, learning either the time-reversal or the bridge with proposed score-matching methods will struggle, since the generated trajectories are highly unlikely to cover this state.
We envisage this problem could be compounded using the adjoints. 
For example, consider conditioning on a state that is highly likely to be reached from a large range of starting states \eg conditioning the Ornstein-Uhlenbeck process on $0$ at time $100$. Then trajectories of the adjoint process will fan out to cover the wide range of likely start states.
This will be a disadvantage, if one is only interested in a small range of specific start points. 
However, it could also be an advantage to some extent, since a larger area will be covered by the data.

We have shown that we can use adjoint processes
to learn the score term $\nabla \log p(t, x; T, y)$ for
given $T, y$, which we use to sample from the bridged diffusion processes. This seems to
perform well compared
to using MCMC methods to
sample trajectories. 
When compared to the method in \citet{heng_simulating_2022} for simulating
forward-in-time bridges, we have demonstrated
that using the adjoints can lead to an
estimation with lower error and negates the need for training twice.

\subsubsection*{Acknowledgements}
The work presented in this article was done at the Center for Computational Evolutionary Morphometry and is partially supported by the Novo Nordisk Foundation grant NNF18OC0052000, by the Chalmers AI Research Centre (CHAIR) (``Stochastic Continuous-Depth Neural Network'') as well as a research grant (VIL40582) from VILLUM FONDEN and UCPH Data+ Strategy 2023 funds for interdisciplinary research.

\bibliography{bib.bib}

\section*{Checklist}



 \begin{enumerate}

 \item For all models and algorithms presented, check if you include:
 \begin{enumerate}
   \item A clear description of the mathematical setting, assumptions, algorithm, and/or model. Yes: all theorems contain a statement of the assumptions and settings.
   \item An analysis of the properties and complexity (time, space, sample size) of any algorithm. Yes: We include the time complexity for the algorithm.
   \item (Optional) Anonymized source code, with specification of all dependencies, including external libraries. Yes: the anonymized source code is included with a read me file containing details on how to install everything.
 \end{enumerate}

 \item For any theoretical claim, check if you include:
 \begin{enumerate}
   \item Statements of the full set of assumptions of all theoretical results. Yes: We include these in the statements of the theorems
   \item Complete proofs of all theoretical results. Yes: The full proof are included in the appendix
   \item Clear explanations of any assumptions. Yes: the assumptions or where to find them are included in the theorem statements.     
 \end{enumerate}

 \item For all figures and tables that present empirical results, check if you include:
 \begin{enumerate}
   \item The code, data, and instructions needed to reproduce the main experimental results (either in the supplemental material or as a URL). Yes: To reproduce the experiments, we have an experiment folder in our codebase with all the code needed to reproduce. 
   \item All the training details (e.g., data splits, hyperparameters, how they were chosen). Yes: We include the hyperparameters in the paper itself, including in the appendix, but the exact setup for the experiments are also in the code base including all hyperparameters.
         \item A clear definition of the specific measure or statistics and error bars (e.g., with respect to the random seed after running experiments multiple times). Yes: where used, these are explained in our experiments section.
         \item A description of the computing infrastructure used. (e.g., type of GPUs, internal cluster, or cloud provider). Yes: This can be found in the appendix with the further experiment details.
 \end{enumerate}

 \item If you are using existing assets (e.g., code, data, models) or curating/releasing new assets, check if you include:
 \begin{enumerate}
   \item Citations of the creator If your work uses existing assets. Yes: Code that we use are cited. The code that we ran in the experiments is linked to. 
   \item The license information of the assets, if applicable. Yes: Links to the code are included which contain any license information.
   \item New assets either in the supplemental material or as a URL, if applicable. Not Applicable
   \item Information about consent from data providers/curators. Not Applicable
   \item Discussion of sensible content if applicable, e.g., personally identifiable information or offensive content. Not Applicable
 \end{enumerate}

 \item If you used crowdsourcing or conducted research with human subjects, check if you include:
 \begin{enumerate}
   \item The full text of instructions given to participants and screenshots. Not Applicable
   \item Descriptions of potential participant risks, with links to Institutional Review Board (IRB) approvals if applicable. Not Applicable
   \item The estimated hourly wage paid to participants and the total amount spent on participant compensation. Not Applicable
 \end{enumerate}

\end{enumerate}

\newpage

\appendix
\onecolumn
\aistatstitle{Score matching for bridges without time-reversals}

\section{PROOFS}\label{Appendix: Proofs}

\subsection{Proof of \cref{thm: loss function}}

\begin{theorem*}
Let $X$ be a diffusion process as defined in \cref{eq: unconditioned process} and suppose further that $f, \sigma$ are $C^2$ and that $X$ admits $C^2$ transition densities. Let $\Bar{X}$ be the conditioned diffusion satisfying \cref{eq: conditioned process}. Let $\Bar{\pbb}$
be the path probability of 
$\Bar{X}$ and 
$\mathbb{Q}^s$ the path probability of 
$X^s$ in \cref{eq: X^s} for a function $s$. 
Then for a given starting point $x_0$
\begin{align}
    \arg\inf_{s} \text{KL}(\Bar{\pbb} \parallel \mathbb{Q}^s) = \arg\inf_{s} \mathcal{L}(s, x_0),
\end{align}
where the infimum is over functions 
$s\colon [0, T]\times \rbb^d \to \rbb^d$.
\end{theorem*}

\begin{proof}\label{proof: kl loss function}
    The Kullback-Leibler divergence is defined by
\begin{align}
    \text{KL}(\Bar{\pbb}\parallel \mathbb{Q}^s)
    =\int \log \frac{\mathrm{d}\Bar{\pbb}}{\mathrm{d}\mathbb{Q}^s}\mathrm{d}\Bar{\pbb}(X).
\end{align}
By Girsanov's theorem, \cite[Theorem 18.8]{schilling_brownian_2014} 
\begin{align}
    \frac{\mathrm{d}\Bar{\pbb}}{\mathrm{d}\mathbb{Q}^s}(X) 
    = \exp\left(
    -\frac{1}{2} \int_0^T \|\sigma^\top(s- s_\theta)(t, X_t)\|^2\mathrm{d}t
    + \int_0^T \sigma^\top(s-s_\theta)(t, X_t) \mathrm{d}B_t
    \right).
\end{align}
Since the expectation of a martingale is equal to $0$, we see that
\begin{subequations}    
\begin{align}
\text{KL}(\Bar{\pbb}\parallel \mathbb{Q}^s)&= 
\int_{X\in \mathcal{X}}
\left[\frac{1}{2}\int_0^T 
\|(s-s_\theta)(t, X_t)\|_{\sigma\sigma^\top}^2 \mathrm{d}t\right]
\mathrm{d}\mathbb{P}^y(X)\\
&=
\frac{1}{2}\int_0^T \Bar{\ebb}^y\left[
\|(s-s_\theta)(t, X_t)\|_{\sigma\sigma^\top}^2 \right]\mathrm{d}t\\
&= \frac{1}{2}\int_0^T \ebb \left[
\|(s-s_\theta)(t, X_t)\|_{\sigma\sigma^\top}^2  \mid X_{0, x_0}(T) = y\right]\mathrm{d}t.
\end{align}
\end{subequations}

By \cite[Theorem 3.7]{bayer_simulation_2013}, this is equal to
\begin{align}
 \frac{1}{2p(0, x_0; T, y)}\int_0^T \lim_{\epsilon \to 0}
\ebb [
\|(s-s_\theta)(t, Y(T-t))\|_{\sigma\sigma^\top}^2 \mathcal{Y}(T) K_\epsilon(Y(T), x_0)]\mathrm{d}t.
\end{align}
We next focus on the contents of the expectation. Writing $s$ out in full as $s=\nabla_x \log p(t, x; T, y)$, using $K_{\epsilon}$ in place of $K_{\epsilon}(Y(T), x_0)$ and expanding the square, we see
\begin{subequations}
\begin{align}
\int_0^T \ebb [
\|(s-&s_\theta)(t, Y(T-t))\|_{\sigma\sigma^\top}^2 \mathcal{Y}(T) K_\epsilon(Y(T), x_0)] \mathrm{d}t\label{eq: pre expanded square}\\
&=\ebb\left[\int_0^T\|s_{\theta}(t, Y(T-t))\|_{\sigma\sigma^\top}^2\mathcal{Y}(T)K_\epsilon \mathrm{d}t \right]\\
&+ \ebb\left[\int_0^T\|\nabla \log p(t, Y(T-t); T, y)\|_{\sigma\sigma^\top}^2 \mathcal{Y}(T) K_\epsilon \mathrm{d}t \right]\\
&-2\mathbb{E}\left[\sum_{l=1}^L \int_{t_{\ell-1}}^{t_{\ell}}
\langle s_{\theta}(t, Y(T-t)), \nabla \log p(t, Y(T-t); T, y)\rangle_{\sigma\sigma^\top} \cdot \mathcal{Y}(T)K_\epsilon\mathrm{d}t\right].\label{eq: mult term}
\end{align}
\end{subequations}
\pagebreak

We focus on the third term, \cref{eq: mult term}. 
Multiplying this out further and using \citet[Theorem 3.6]{bayer_simulation_2013}, we see that
\begin{subequations}
\begin{align}
    &\mathbb{E}\left[
    s_{\theta}(t, Y(T-t))\cdot \nabla \log p(t, Y(T-t); T, y) \cdot \mathcal{Y}(T)\right]\\
    = &\int s_{\theta}(t, x)\nabla_x \log p(t, x, T, y) p(t_0, x_0; t, x) p(t, x; T, y) \mathrm{d}x_0 \mathrm{d}x\\
    = & \int s_\theta(t, x)\nabla_x p(t, x; T, y) p(t_0, x_0; t, x) \mathrm{d}x_0 \mathrm{d}x.\label{eq: before chapman}
\end{align}
\end{subequations}

Using Chapman-Kolmogorov it holds that for any $t<t_1<T$,
\begin{align}
    p(t, x; T, y) = \int p(t, x; t_1, x_1)p(t_1, x_1; T, y)\mathrm{d}x_1,
\end{align}
and so, using the properties of the differential of the logarithm it holds
\begin{align}
    \nabla_x p(t, x; T, y) = \int \nabla_x \log p(t, x; t_1, x_1) p(t, x; t_1, x_1)p(t_1, x_1; T, y)\mathrm{d}x_1.
\end{align}
Putting this back into \cref{eq: before chapman} and using \citet[Theorem 3.6]{bayer_simulation_2013} we get,
\begin{subequations}
\begin{align}
    &\int [s_{\theta}(t, x) \nabla_x \log p(t, x; t_1, x_1)] p(t_0, x_0; t, x) p(t, x; t_1, x_1)p(t_1, x_1; T, y) \mathrm{d}x_0 \mathrm{d}x \mathrm{d}x_1\\ 
=&\ebb[s_{\theta}(t, Y(T -t))\nabla \log p(t, Y(T-t); t_1, Y(T-t_1))\mathcal{Y}(T)].
\end{align}
\end{subequations}

From this we see that \cref{eq: mult term} becomes
\begin{align}
    -2 \mathbb{E}\left[
    \sum_{\ell=1}^L \int_{t_{\ell-1}}^{t_\ell}
    \langle s_\theta(t, Y(\hat{t})),
    \nabla \log p(t, Y(\hat{t}); t_\ell, Y(\hat{t_\ell}))\rangle_{\sigma\sigma^\top} \mathcal{Y}(T)K_\epsilon \mathrm{d}t
    \right].
\end{align}

Then \cref{eq: pre expanded square} is equal to
\begin{align}
    \sum_{\ell=1}^L \int_{t_{\ell-1}}^{t_\ell} \mathbb{E}\left[
    \|s_\theta(t, Y(\hat{t})) - \nabla \log p(t, Y(\hat{t}); t_{\ell}, Y(\hat{t}_{\ell}))\|_{\sigma\sigma^\top}^2 \mathcal{Y}(T)K_\epsilon \right] \mathrm{d}t + \text{const}.,
\end{align}
where
\begin{align}
    \text{const.} = \sum_{\ell=1}^L \int_{t_{\ell-1}}^{t_\ell} \left[
    \|\nabla \log p(t, Y(\hat{t}); T, y)\|_{\sigma\sigma^\top}^2
    - \|\nabla \log p(t, Y(\hat{t}); t_{\ell}, Y(\hat{t_{\ell}}))\|_{\sigma\sigma^\top}^2 \right] \mathcal{Y}(T)K_\epsilon \mathrm{d}t.
\end{align}

\end{proof}

\subsection{Proof of \cref{thm: independent loss fn}}
\begin{theorem*}\label{proof: x_0 integral}
We assume the same setup as in \cref{thm: loss function}. Define $\Bar{\mathcal{L}}(s) \coloneqq \int\mathcal{L}(s, x_0) \mathrm{d}x_0$. Then, letting $I_\ell$ be defined as in \cref{eq: Il}
\begin{subequations}
\begin{align}
        \Bar{\mathcal{L}}(s) &\coloneqq \mathbb{E} \left[\frac{\mathcal{Y}(T)}{p(0, Y(T); T, y)}
        \sum_{l=1}^L I_\ell
    \right].
\end{align}
\end{subequations}

\end{theorem*}

\begin{proof}
By definition 
\begin{align}
    \int \mathcal{L}(s, x)\mathrm{d}x 
    = \int
      \frac{1}{p(0, x; T, y)}\lim_{\epsilon \to 0}
    \sum_{l=1}^L \int_{t_{l-1}}^{t_l}
    \ebb[
    \|s(t, Y(\hat{t}))
    - \nabla \log p(t, Y(\hat{t}); t_l, Y(\hat{t_l}))\|_{\sigma\sigma^\top}^2
    \mathcal{Y}(T) K_\epsilon(Y(T), x)
    ] \mathrm{d}t
    \mathrm{d}x.
\end{align}

Swapping the order of integration, this is equal to
\begin{align}
    \lim_{\epsilon \to 0}
    \sum_{l=1}^L \int_{t_{l-1}}^{t_l}
    \ebb\left[
    \int \frac{\mathcal{Y}(T) K_\epsilon(Y(T), x)}{p(0, x; T, y)}
    \|s(t, Y(\hat{t}))
    - \nabla \log p(t, Y(\hat{t}); t_l, Y(\hat{t_l}))\|_{\sigma\sigma^\top}^2
    \mathrm{d}x
    \right] \mathrm{d}t.
\end{align}
Now noting that the only term to depend on $x$ is $p(0, x; T, y)$, and that 
\begin{align}
    \lim_{\epsilon\to 0}\int \frac{1}{p(0, x; T, y)}K_\epsilon(Y(T), x) \mathrm{d}x = \frac{1}{p(0, Y(T); T, y)},
\end{align}
gives the result.
\end{proof}

\section{ALGORITHM}
This is the exact algorithm we use to condition a given SDE on a given fixed point $y$, as described in \cref{sec: computing loss}.
\begin{algorithm}[ht]
\caption{SDE Bridge: Condition an SDE on a fixed endpoint $y$}
\label{algo: fixed y}
\SetCustomAlgoRuledWidth{0.45\textwidth}
\SetKwInput{Input}{input}
\Input{Number of batches $N$, endpoint $y$, time grid $\{t_\ell\}_{\ell=1}^L$}
$\hat{t}_\ell \gets T - t_{L-\ell}$
\\
\While{not converged}{
For given $y$, compute $N$ trajectories 
$\{Y^y(\hat{t}_\ell)\}_{\ell=1}^L$
\\
$Z^y_{\ell} \gets Y^y(\hat{t}_{L-\ell})$
\\
$g_\ell \gets g(t_\ell, Z^y_{\ell}, t_{\ell+1}, Z^y_{{\ell+1}})$
\\
$\mathcal{L}_\theta \gets 
\frac{1}{(\Delta t) N}
\sum_{\ell=0}^{L-1} \sum_{Z^y_{\ell}} 
\|[s_\theta(t_\ell, Z^y_{\ell})
- g_\ell]\|_{\sigma\sigma^\top}^2$
\\
Compute $\frac{\mathrm{d} \mathcal{L}_{\theta}}{d\theta}$ and use it to update $s_\theta$
}
\end{algorithm}

For a batch of $N$ trajectories, and $L$ time points per trajectory, with the SDE having dimension $d$ the computational complexity is $O(N\cdot L \cdot d^3).$ The $d^3$ complexity comes from the covariance matrix arithmetic.

\section{EXPERIMENT DETAILS}\label{Appendix: experiments}

Here we provide some more details for the individual experiment setups.
For all of the experiments one NVIDIA RTX 4090 GPU and one Intel(R) Xeon(R) CPU E5-2650 v4 @ 2.20GHz were used.

\subsection{Ornstein-Uhlenbeck}\label{Appendix: experiments ou}

We use the same Ornstein-Uhlenbeck formulation for both experiments as detailed in \cref{eq: OU unconditioned}. The adjoint processes are explicitly computed as follows:
\begin{subequations}   
\begin{align}
    &\mathrm{d}Y(t) = Y(t) \mathrm{d}t + \mathrm{d}W(t) \quad &Y(0)=y\\
    &\mathrm{d}\mathcal{Y}(t) = \mathcal{Y}(t) \mathrm{d}t \quad & \mathcal{Y}(0)=1.
\end{align}
\end{subequations}
Moreover, the score function for the unconditioned process $X_t$ is given by
\begin{align}
    \nabla_x \log p(t, x; s, y) = \frac{\exp\{-(s-t)\}}{v(s-t)}\left(
        y - m(s-t, x)
    \right),
\end{align}
where the functions $m, v$ are the mean and variance defined as
\begin{align}
    m(t, x) = x \exp(-t) \qquad v(t) = \frac{1-\exp(-2 t)}{2}.
\end{align}

\paragraph{Variable \texorpdfstring{$y$}{y}}\label{appendix: ou variable y}
The values $y$ are sampled uniformly from the interval $[-1, 1]$. We train on $1000$ iterations each with $1000$ sample paths. These parameters are chosen following the similar experiment from \citet{heng_simulating_2022}. For the neural network, we use the same as \citet{heng_simulating_2022}: we use the sinusoidal embedding to encode the time steps \citep{vaswani2017attention}, and use three multi-layer perceptron blocks with the Leaky ReLU activation function. In \cref{fig: ou data distribution} we plot the generated data at three different time steps, and in \cref{fig: ou error} we plot the absolute error between the true score and learned score trained on said generated data. 
For smaller times, the data has not yet had so much time to disperse leading to higher errors in areas where there is a larger difference between $x$ and $y$. 

\paragraph{Distributed \texorpdfstring{$y$}{y}}
\begin{figure}
    \centering
    \includegraphics{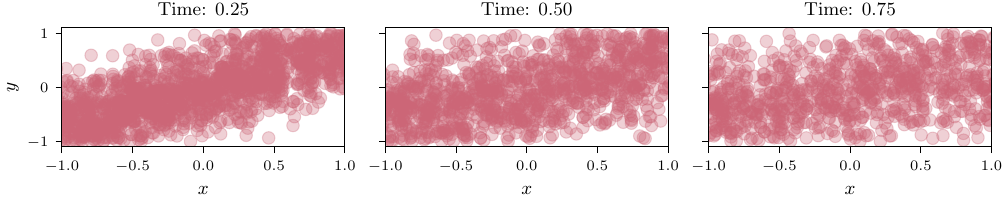}
    \caption{For the time points $t \in [0.25, 0.50, 0.75]$ we plot all the points $Y(t)$ of the generated adjoints that were used in training to learn the score $\nabla_x \log p(t, x; 1, y)$.  }
    \label{fig: ou data distribution}
\end{figure}

\begin{figure}
    \centering
\includegraphics{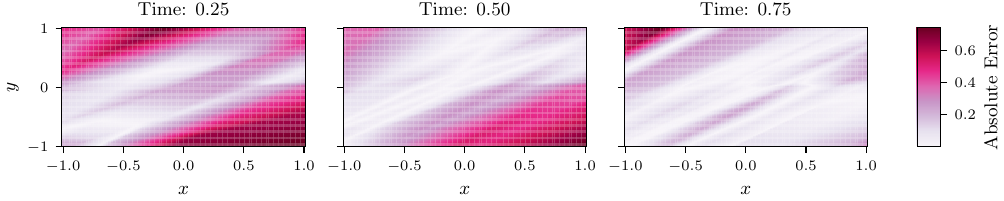}
    \caption{We plot the absolute error between the true score and the learned score from \cref{fig:varied y error}.}
    \label{fig: ou error}
\end{figure}

For all methods we train on 100 iterations of 1000 samples, and use the same MLP network architecture used by \citet{heng_simulating_2022}, as described in \cref{appendix: ou variable y}. 
We train with five different random seeds and plot the average mean square error and the standard deviation between the different trainings in \cref{fig:dim_comparison}.

\subsection{Cell model}
\begin{figure}[ht]
    \centering
\includegraphics{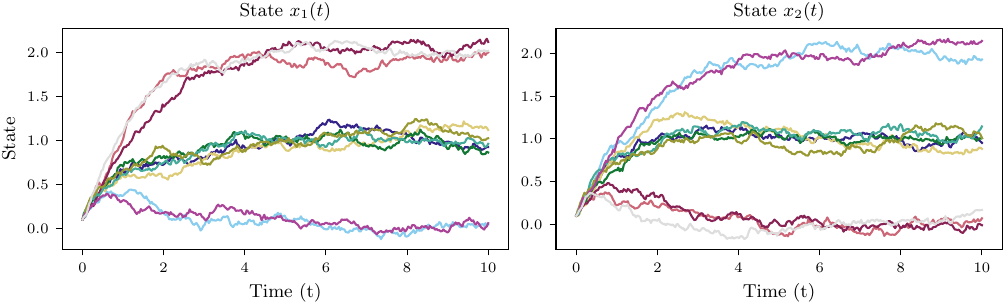}
    \caption{Trajectories from the unconditioned process of cell differentiation into three different states.}
    \label{fig: cell forward}
\end{figure}

In \cref{fig: cell forward} we plot unconditional trajectories from the cell model defined in \cref{eq: cell sde}.
For the proposed method, the time-reversed bridge (DB (bw)) and the forward bridge (DB (fw)), we train on $10,000$ trajectories in total with batch sizes of $100$. We use the same network structure for both: again the architecture described in \cref{appendix: ou variable y}.
However note the proposed method is written in Jax \citep{jax2018github}, whereas the others are implemented in Pytorch \citep{paszke2019pytorch}, since we used the code from the authors\footnote{\texttt{https://github.com/jeremyhengjm/DiffusionBridge}.}.
For the guided method we use \cite{schauer_guided_2017}'s code written in Julia \footnote{Codebase at \texttt{https://github.com/mschauer/Bridge.jl} from file \texttt{project/partialbridge\_fitzhugh.jl}}. We first simulate \textcolor{black}{$1000$} trajectories, without saving, and from then save every \textcolor{black}{$100^{\text{th}}$} accepted trajectory, with $\rho=0.7$. 

\subsection{Shape models}

The SDE in \cref{eq: kunita sde} represents the position of the points $x_i$ at time $t$. We can equivalently look at the SDE representing the change of position, $Y(t) \coloneqq X(t) - x_0$, with the constraint $Y(0) = 0$:
\begin{align}\label{eq: kunita difference}
    Y(t) = \int \sum_{i=1}^m k(Y(t) + x_0, z_i)\mathrm{d}W(t).
\end{align}
In the shape experiment, we condition the process $Y_t$ as \citet{baker_conditioning_2024,yang2024simulating} do. This gives better results in training than conditioning the position $X_t$ (\cref{eq: kunita sde}) directly.

We train on a total of $409,600$ trajectories with a batch size of $128$. 
We set $\sigma=0.04$ and $\alpha=1/\sigma$ for the Gaussian function $k$ defining the SDE. 
For the grid points $\{z_i\}_{i=1}^m$ we take a $50\times 50$ grid of evenly spaced values on $[-1.5, 1.5]\times [-1.5, 1.5]$. 
We train on $30$ total points and sample $80$ points after training. 
We use the neural network architecture proposed by \cite{yang2024simulating}: a U-Net based Fourier neural operator architecture.

\end{document}